\documentclass{article}

\usepackage[utf8]{inputenc} 
\usepackage[T1]{fontenc}    
\usepackage{hyperref}       
\usepackage{url}            
\usepackage{booktabs}       
\usepackage{amsfonts}       
\usepackage{nicefrac}       
\usepackage{microtype}      
\usepackage{amsmath}
\usepackage[linesnumbered, ruled]{algorithm2e}

\usepackage[round,sectionbib]{natbib}
\usepackage{graphicx}
\usepackage{lineno}
\usepackage{setspace}

\usepackage{microtype}
\usepackage{graphicx}
\usepackage{subfigure}
\usepackage{booktabs} 
\usepackage{xcolor}
\usepackage{bbm}
\usepackage{mathtools}
\usepackage{verbatim}
\usepackage{enumitem}
\setlist[itemize]{topsep=0pt, leftmargin=3mm}

\newcommand{\E}{\mathbf{E}}
\newcommand{\real}{\mathbb{R}}

\usepackage{multicol}
\usepackage{wrapfig, blindtext}
\usepackage{amsthm}

\title{Regret Balancing for Bandit and RL Model Selection}

\author{%
  Yasin Abbasi-Yadkori \\
  DeepMind\\
  \texttt{yadkori@google.com} \\
  \and
  Aldo Pacchiano \\
  UC Berkeley \\
  \texttt{pacchiano@berkeley.edu} \\
  \and
  My Phan \\
  University of Massachusetts \\
  \texttt{myphan@cs.umass.edu} \\
}

\usepackage{tikz}

\newtheorem{theorem}{Theorem}[section]
\newtheorem{corollary}{Corollary}[section]
\newtheorem{lemma}[theorem]{Lemma}

\allowdisplaybreaks

\begin{document}
\maketitle

\begin{abstract}
We consider model selection in stochastic bandit and reinforcement learning problems. Given a set of base learning algorithms, an effective model selection strategy adapts to the best learning algorithm in an online fashion. We show that by estimating the regret of each algorithm and playing the algorithms such that all empirical regrets are ensured to be of the same order, the overall regret balancing strategy achieves a regret that is close to the regret of the optimal base algorithm. Our strategy requires an upper bound on the optimal base regret as input, and the performance of the strategy depends on the tightness of the upper bound. We show that having this prior knowledge is necessary in order to achieve a near-optimal regret. Further, we show that any near-optimal model selection strategy implicitly performs a form of regret balancing.
\end{abstract}

\section{Introduction}

We study the problem of choosing among a set of learning algorithms in sequential decision-making problems with partial feedback. Learning algorithms are designed to perform well when certain favorable conditions are satisfied. However, the learning agent might not know in advance which algorithm is more appropriate for the current problem that the agent is facing. 

As an example, consider the application of stochastic bandit algorithms in personalization problems, where in each round a user visits the website and the learning algorithm should present the item that is most likely to receive a click or be purchased. When contextual information (such as location, browser type, etc) is available, we might decide to learn a click model given the user context. If the context is not predictive of the user behavior, using a simpler non-contextual bandit algorithm might lead to a better performance. As another example, consider the problem of tuning the exploration rate of bandit algorithms. Typically, the exploration rate in an $\epsilon$-greedy algorithm has the form of $c/t$, where $t$ is time and the optimal value of constant $c$ depends on unknown quantities related to reward vector. The decision rule of the UCB algorithm also involves an exploration bonus~\citep{ACF-2002}. Choosing values smaller than the theoretically suggested value can lead to better performance in practice if the theoretical value is too conservative. However, if the exploration bonus is too small, the regret can be linear. It is desirable to have a model selection strategy that finds a near-optimal parameter value in an online fashion. 

A model selection strategy can also be useful in finding effective reinforcement learning methods. There has been a great number of reinforcement learning algorithms proposed and studied in the literature~\citep{SB-2018, Szepesvari-2010}. In some specialized domains, we might have a reasonable idea of the type of solution that can perform well. In general, however, designing a reinforcement learning solution can be a daunting task as the solution often involves many components. 
In fact, in some problems it is not even clear if we should use a reinforcement learning solution or a simpler contextual bandit solution. 
For example, bandit algorithms are used in many personalization and recommendation problems, although the decisions of the learning system can potentially change the future traffic and inherently we face a Markov decision process. In such problems, the available data might not be enough to solve the problem using an RL algorithm and a simpler bandit solution might be preferable. The complexity of the RL problem is often not known in advance and we would like to adapt to the complexity of the problem in an online fashion. 

While model selection is a well-studied topic in supervised learning, results in the bandit and RL setting are scarce. \cite{MM-2011} propose a method for the model selection problem based on EXP4 with additional uniform exploration. \cite{ALNS-2017} obtain improved results by an online mirror descent method with a carefully selected mirror map. The algorithm is called \textsc{CORRAL}, and under a stability condition, it is shown to enjoy strong regret guarantees. Many bandit algorithms that are designed for stochastic environments (such as UCB, Thompson sampling, etc) do not satisfy the stability condition and thus cannot be directly used as base algorithms for CORRAL. Although it might be possible to make these algorithm stable by proper modifications, the process can be tedious. To overcome this issue, \cite{PPARZLS-2020} propose a generic smoothing procedure that transforms nearly any stochastic algorithm into one that is stable. Results of \cite{ALNS-2017} and \cite{PPARZLS-2020} require the knowledge of the optimal base regret. \cite{FKL-2019} study bandit model selection among linear bandit algorithms when the dimensionality of the underlying linear reward model, and thus the optimal base regret, is not known. A related problem is studied by \cite{CMB-2020}.  

In this paper, we propose a model selection method for bandit and RL problems in stochastic environments. We call our method ``regret balancing" because it maintains regret estimates of base algorithms and tries to keep the \textit{empirical regret} of all algorithms roughly the same. The method achieves regret balancing by playing the base algorithm with the smallest empirical regret. An algorithm can have small empirical regret for two reasons: either it chooses good actions, or it has not been played enough. By playing the algorithm with the smallest empirical regret, the model selection procedure finds an effective trade-off between exploration and exploitation.

The proposed approach has several notable properties. First, no stability condition is needed and any base algorithm without any modifications can be used. Note that when applied to stochastic bandit algorithms, \cite{ALNS-2017} and \cite{PPARZLS-2020} modify the base algorithms to ensure certain stability conditions. Second, our approach is intuitive and almost as simple as a UCB rule. By contrast, many existing model selection approaches have a complicated form. Finally, the approach can be readily applied to reinforcement learning problems. 

The proposed approach, similar to a number of existing solutions, requires the knowledge of the regret of the optimal base algorithm. We show that, in general, any model selection strategy that achieves a near-optimal regret requires either the optimal base regret or direct sampling from the arms. We show that by adding a forced exploration scheme, and hence direct access to the arms, the regret balancing strategy can achieve near-optimal regret in a class of problems without the knowledge of the optimal base regret. 
Further, we show a class of problems where any near-optimal model selection procedure is indeed implementing a regret balancing method, possibly implicitly. 

As we will show, the regret of our model selection strategy is $\Omega(T)$, where $T$ is time horizon. This regret is minimax optimal, given the existing lower bound for the model selection problem that scales as $\Omega(\sqrt{T})$~\citep{PPARZLS-2020}; Even if it is known that a base algorithm has logarithmic regret, the fast logarithmic regret cannot be preserved in general. 

We show a number of applications of the proposed approach for model selection. We show how a near-optimal regret can be achieved in the class of $\epsilon$-greedy algorithms without any prior knowledge of the reward function. We also show how the proposed approach can be used for representation learning in bandit problems. Further, we show a model selection strategy to choose among reinforcement learning algorithms. As a consequence for reinforcement learning, if a set of feature maps are given and the value functions are known to be linear in a feature map belonging to this set, we can use the regret balancing strategy to achieve a regret that is near-optimal up to a constant factor. Finally, the proposed regret balancing strategy can also be used as a bandit algorithm. We show how the approach is implemented as an algorithm for linear stochastic bandits.

\subsection{Problem Definition}

For an integer $A$, we use $[A]$ to denote the set $\{1,2,\dots,A\}$. A contextual bandit problem is a sequential game between a learner and an environment. We consider a set of learners $[M]$. The game is specified by a context space $S$, an action set $[K]$ of size $K$, a reward function $r:S\times [K]\rightarrow [0,1]$, and a time horizon $T$. In round $t\in [T]$, the learner $i\in [M]$ observes the context $s_t\in S$ and chooses an action $a_t\in [K]$ from the action set. Then the learner observes a reward $r_{t} = r(s_t,a_t)+\eta_t$, where for a positive constant $\sigma$, $\eta_t$ is a $\sigma$-sub-Gaussian random variable, meaning that for any $\lambda\in\real$, $\E[e^{\lambda \eta_t}] \le e^{\lambda^2 \sigma^2/2}$. In the special case of linear contextual bandits~\citep{LS-2020}, we are given a feature map $\phi:S\times [K]\rightarrow \real^d$ such that $r(s,a) = \phi(s,a)^\top \theta_*$ for an unknown vector $\theta_*\in\real^d$. Let $\mu_{*,t}= \E(\max_a r(s_t,a))$ be the expected reward of the optimal action at time $t$, where expectation is taken with respect to the randomization in $s_t$ and $\eta_t$. The goal is to have small regret, defined as $C_{i,T} = \sum_{t=1}^T (\mu_{*,t} - r_{t})$. If $\{s_t\}_{t=1}^T$ is an IID sequence, then $\mu_{*,t}$ is the same constant for all rounds and we use $\mu_*$ to denote this value. The game is challenging as the reward function is not known in advance. If $S$ contains only one element, then the problem reduces to the multi-armed bandit problem. If an action influences the distribution of the next context, then the problem is a Markov decision process (MDP) and it is more suitable to define regret with respect to the \textit{policy} that has the highest total (or stationary) reward (See Section~\ref{section:applications} for more details).   

A bandit model selection problem is specified by a class of bandit problems and a set of bandit algorithms. Let $M$ be the number of bandit algorithms (called base algorithms in what follows). As defined above, $C_{i,T}$ is the regret of the $i$th base in the underlying bandit problem if the base algorithm is executed alone. In a bandit model selection problem, the decision making is a two step process. In round $t$, the learner choose base $i_t$ from the set of $M$ bandit algorithms, the base observes the context $s_t$ and selects an action $a_t$ from the set of $K$ actions, and the reward $r_t$ of the action is revealed to the learner. Then the internal state of the base $i_t$ is updated using reward $r_t$. The regret of the overall model selection strategy is  defined with respect to $\mu_{*,t}$:
\[
\text{Regret}_T = \sum_{t=1}^T (\mu_{*,t} - r_{t}) \;.
\]
Let $i_*$ be the optimal base with the smallest regret if it is played in all rounds, $i_* = \arg\min_i C_{i,T}$. We would like to ensure that $\text{Regret}_T = O(C_{i_*,T})$. A reinforcement learning model selection problem is defined similarly (See Section~\ref{section:applications} for more details). 


\section{Regret Balancing}

At a high level, the main idea is to estimate the empirical regret of the base algorithms during the rounds that the algorithms are played, and ensure that all base algorithms suffer roughly the same empirical regret. This simple idea ensures a good trade-off between exploration and exploitation: if a base algorithm is played only for a small number of rounds, or if it plays good actions, then its empirical regret will be small and will be chosen by the model selection procedure. 

\subsection{Bandit Model Selection}
\label{section:bandit}


In this section, we present the regret balancing model selection method. Consider a bandit model selection problem in a stochastic environment. Let $N_{i,t}$ be the number of rounds that base $i$ is played up to but not including round $t$, and let $R_{i,t}$ be the total reward of this base during these $N_{i,t}$ rounds. With an abuse of notation we also use $N_{i,t}$ to denote the set of rounds that base $i$ is selected. Let $S_{i,t}$ be all data in the rounds that base $i$ is played, $S_{i,t} = \{(s_t,a_t,r_t)\,:\, t\in N_{i,t}\}$. Let $\mathbb{H}$ be the space of all such histories for all $i$ and $t$. We use $R_{*,t}$, $N_{*,t}$, and $S_{*,t}$ to denote the quantities related to the optimal base, which was defined earlier in the problem definition. Regret of base $i$ during the $N_{i,t}$ rounds is $G_{i,t} = \sum_{\tau\in N_{i,t}} \mu_{*,\tau} - R_{i,t}$. We assume that a high probability (possibly data-dependent) upper bound on the regret of the optimal base algorithm is known: a function $U:\real\times \mathbb{H}\rightarrow \real$ is given so that for any $\delta\in (0,1)$, with probability at least $1-\delta$, $G_{i_*,t} \le U(\delta, S_{*,t})$ for any $t$.\footnote{We can use different probabilistic guarantees here, and any form used here will also appear in Theorem~\ref{theorem:model-selection-bandit}.} For example, for the UCB algorithm we have $U(\delta, S_{*,t}) = \widetilde O(\sqrt{K t \log(1/\delta)})$,\footnote{We use $\widetilde O$ notation to hide polylogarithmic terms.} and for the OFUL algorithm we have $U(\delta, S_{*,t}) = \widetilde O(\log(\det(V_t)/\delta)\sqrt{t})$, where $V_t$ is an empirical covariance matrix~\citep{APS-2011}. 
Given that $G_{i_*,t}$ is defined with respect to the realized rewards $R_{i_*,t}$, the regret upper bound $U$ should be at least of order $\Omega(\sqrt{t})$.  

Next, we describe the model selection strategy. In round $t$, let $j_t$ be the \textit{optimistic} base and $b_t$ be the optimistic value,
\begin{equation}
\label{eq:emp_regret_k}
j_t = \arg\max_{i\in [M]}\, \frac{R_{i,t}}{N_{i,t}} + \frac{ U(\delta, S_{i,t})}{N_{i,t}}\,,\qquad b_t = \frac{R_{j_t,t}}{N_{j_t,t}} + \frac{ U(\delta, S_{j_t,t})}{N_{j_t,t}} \;.
\end{equation}
Variable $b_t$ estimates the value of the best action. Define the empirical regret of base $i$ by
\[
\widehat G_{i,t} = N_{i,t} b_t - R_{i,t} \;.
\]
Recall the true regret defined by $G_{i,t} = \sum_{\tau\in N_{i,t}} \mu_{*,\tau} - R_{i,t}$. Notice that we have $N_{j_t,t} b_t - R_{j_t,t} = U(\delta, S_{j_t,t})$, i.e. $b_t$ is chosen so that the empirical regret of the optimistic base scales as the target regret of the optimal base. Throughout the game, we play bases to ensure that the empirical regrets of all bases are roughly the same. To be more precise, in time $t$, we choose the base with the smallest empirical regret:
\[
i_t = \arg\min_{i\in [M]}\, \widehat G_{i,t} \;.
\]
This choice will most likely increase the empirical regret of base $i_t$. 
Next theorem shows the model selection guarantee of the regret balancing strategy. 
\begin{theorem}
\label{theorem:model-selection-bandit}
If $\mu_{*,t}=\mu_*$ for a constant $\mu_*$ regardless of time $t$, and if with probability at least $1-\delta$, $G_{i_*,t} \le U(\delta, S_{i_*,t})$ for any $t$, then $\text{Regret}_T \le M \max_i U(\delta, S_{i,T})$ with probability at least $1-\delta$.
\end{theorem}
\begin{proof}
First, we show that $b_t$ is an optimistic estimate of the average optimal reward. By \eqref{eq:emp_regret_k} and the regret guarantee of the optimal base,
\begin{align}
\label{eq:bt_k}
b_t = \frac{R_{j_t,t}}{N_{j_t,t}} + \frac{ U(\delta, S_{j_t,t})}{N_{j_t,t}} \ge \frac{R_{*,t}}{N_{*,t}} + \frac{ U(\delta, S_{*,t})}{N_{*,t}} \ge \frac{\sum_{\tau\in N_{*,t}} \mu_{*,\tau}}{N_{*,t}} = \mu_*  \;.
\end{align}
Let $i_t$ be the base chosen at time $t$ and $j_t$ be the optimistic base. The cumulative regret of base $i_t$ at time $t$ can be bounded as
\begin{align}
\label{eq:one_base_regret2}
\notag
G_{i_t,t} &= N_{i_t,t} \mu_{*} - R_{i_t,t} \\
\notag
&\le N_{i_t,t} b_{t} - R_{i_t, t} &\text{By \eqref{eq:bt_k}}  \\
\notag
&\le N_{{j_{t}},t} b_{t} - R_{{j_t}, t} &\text{By definition of $i_t$} \\
&= U(\delta, S_{j_{t},t}) \;. &\text{By definition of $j_t$ and $b_t$}
\end{align}
Let $T_i$ be the last time step that base $i$ is played. Given that the instantaneous regret is upper bounded by $1$,  by~\eqref{eq:one_base_regret2} the regret can be bounded as
\begin{align*}
\sum_{i=1}^M G_{i,T} &= \sum_{i=1}^M G_{i,T_i} \le \sum_{i=1}^M U(\delta, S_{j_{T_i}, T_i})\le  M \max_i U(\delta, S_{i,T}) \;.
\end{align*}

\end{proof}

The condition that $\mu_{*,t}=\mu_*$ for a constant $\mu_*$ regardless of time $t$ is needed to ensure that $b_t \ge \sum_{\tau\in N_{i,t}} \mu_{*,\tau}/N_{i,t}$ for any base $i$. The condition holds in the following model selection problems: choosing a feature mapping in a stochastic bandit problem, and choosing the optimal exploration rate among a number of $\epsilon$-greedy algorithms. The condition is also satisfied for choosing between multi-armed bandits and stochastic linear contextual bandits, where $\mu_{*,t} = \mathbf{E} (\max_{i\in [K]} \phi(s_t,i)^\top \theta_*)$ is a time-independent constant value for IID context $s_t$.


As we mentioned earlier, the regret upper bound $U$ should be of order $\Omega(\sqrt{T})$. Thus, our approach can achieve the regret of the optimal base as long as the optimal regret is at least $\Omega(\sqrt{T})$. This observation is consistent with the lower bound argument of \cite{PPARZLS-2020} who show that, in general, $O(\sqrt{T})$ is the best rate that can be achieved by any model selection strategy. Unfortunately, this lower bound implies that in a model selection setting, we can no longer hope to achieve the logarithmic regret bounds that can be usually obtained in stochastic bandit problems. Notice that such logarithmic bounds are shown for the \textit{pseudo-regret} and not for the regret as defined above. The pseudo-regret is the difference of the expected rewards of the optimal arm and the arm played, and is not directly observed by the learner, and it can be estimated only up to an error of order $\Omega(\sqrt{T})$.

\subsection{Applications}
\label{section:applications}

In this section, we show some applications of the regret balancing strategy. 

\subsubsection*{Regret Balancing for Bandits}
\label{section:linear_bandit}

The regret balancing strategy can be used as a bandit algorithm. To use as a multi-armed bandit algorithm, we treat each arm as a base algorithm and we choose $U(\delta, t)=\sqrt{(t/2) \log(1/\delta)}$ as the regret of the optimal arm $a_*$. To see this, notice that by the sub-Gaussianity of the noise, with probability at least $1-\delta$, $G_{a_*,t} = \sum_{\tau\in N_{a_*,t}} \mu_{*} - R_{a_*,t} = \sum_{\tau\in N_{a_*,t}} (\mu_{*} - \mu_* + \eta_t)  = \sqrt{(t/2) \log(1/\delta)}$. In Figure~\ref{fig:RM_vs_UCB}-Left, we compare regret balancing with the UCB algorithm~\citep{ACF-2002} on a 4-armed Bernoulli bandit with means $\{0.1, 0.2, 0.3, 0.4\}$. 
In regret balancing, we treat each arm as a base algorithm and so we use $U(t)=\sqrt{(t/2) \log(1/\delta)}$ with $\delta=0.1$ as the target regret.  

\begin{wrapfigure}{}{0.8\textwidth}
\includegraphics[width = 0.5\linewidth]{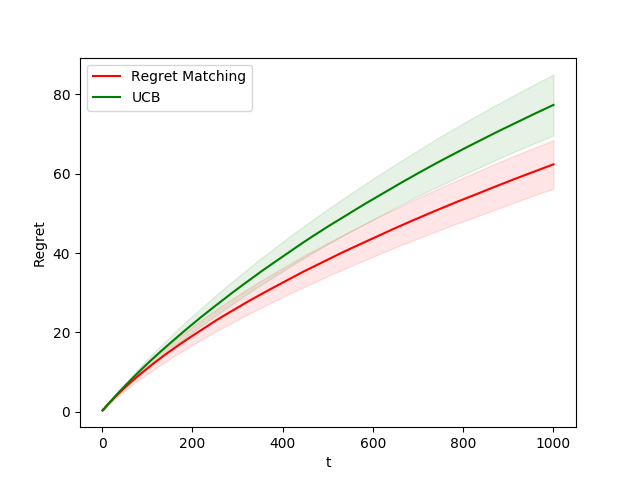} 
\includegraphics[width = 0.5\linewidth]{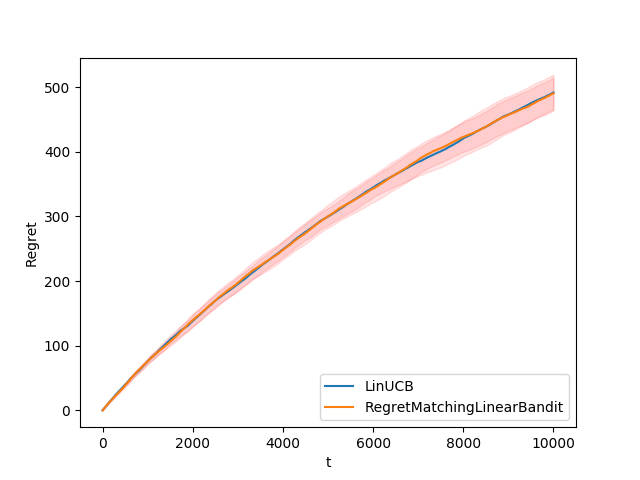} 
\caption{Regret Balancing vs UCB and OFUL. Mean and standard deviation of 2000 and 20 runs.    }
    \label{fig:RM_vs_UCB}
\end{wrapfigure}

Next, we show the implementation of the strategy as an algorithm for the linear stochastic bandits. Consider the following problem. In round $t$, the learner chooses action $x_t$ from a (possibly time varying) decision space that is a subset of the unit sphere $D_t\subset \mathbb{S}^d$ and observes a reward $y_t = x_t^\top \theta_* + \eta_t$, where $\theta_*\in \real^d$ is an unknown parameter vector and $\eta_t$ is a $\sigma$-sub-Gaussian noise term.\footnote{This formulation includes the special case of linear contextual bandits with $D_t = \{\phi(s_t,a)\,:\, a\in [K]\}$.} Let $x_{t,*}$ be the optimal action at time $t$ defined as $x_{t,*} = \arg\max_{x\in D_t} x^\top \theta_*$. 
The objective is to have small regret defined as $\text{Regret}_T = \sum_{t=1}^T (x_{t,*}^\top \theta_* - x_t^\top \theta_*)$. 

We state some notation before defining the bandit method. For a regularization parameter $\lambda>0$, let $V_t = \lambda I +  \sum_{k=1}^{t-1} x_k x_k^\top$ be the empirical covariance matrix, and let $\|z\|_V = \sqrt{z^\top V z}$ be the weighted $\ell^2$-norm of vector $z$. Let $\widehat\theta_t = V_t^{-1}\sum_{k=1}^{t-1} x_k y_k$ be the regularized least-squares estimate. Let $\beta_t(\delta)=O(\sqrt{d \log(t)})$ be as defined in Appendix~\ref{app:useful}. Let $y_t = \arg\max_{x\in D_t}\, x^\top \widehat \theta_t + \beta_t(\delta) \|x\|_{V_t^{-1}}$ 
be the ``optimistic" choice in round $t$. A UCB approach would take action $y_t$ next. Regret balancing, however, uses the optimistic choice to estimate the empirical regrets of different choices. Let  $b_t = y_t^\top \widehat \theta_t + \beta_t(\delta) \|y_t\|_{V_t^{-1}}$, which will be shown to be an upper bound on the value of the best action. In time $t$, we choose the action with the smallest \textit{empirical regret}, 
\[
x_t = \arg\min_{x\in D_t}\, \widehat G_{x,t}\,,\qquad \widehat G_{x,t} = \frac{b_t - x^\top \widehat \theta_t}{\|x\|_{V_t^{-1}}^2}  \;.
\]
Intuitively, $b_t - x^\top \widehat \theta_t$ is an estimate of the instantaneous regret of action $x$ and $1/\|x\|_{V_t^{-1}}^2$ is roughly the number of times that $x$ is played.\footnote{In multi-armed bandits, where actions are fixed axis aligned unit vectors, $1/\|x\|_{V_t^{-1}}^2$ counts the number of times an action is played.} 
Next theorem bounds the regret of the regret balancing strategy. The proof is in Appendix~\ref{app:linear_bandit}.
\begin{theorem}
\label{thm:linear_bandit}
For any $\delta\in (0,1)$, with probability at least $1-\delta$, $\text{Regret}_T = \widetilde O(d^{3/2} \sqrt{T})$. Here $\widetilde O$ hides polylogarithmic terms in $T$, $d$, $\lambda$, and $1/\delta$.
\end{theorem}

The regret bound in the theorem is slightly worse than the minimax optimal rate of $\widetilde O(d \sqrt{T})$, however and as we show next, regret balancing strategy can be a competitive linear bandit algorithm in practice. In Figure~\ref{fig:RM_vs_UCB}-Right, we compare regret balancing, as described above, with the OFUL algorithm~\citep{APS-2011} on a contextual linear bandit problem with two arms: for $i\in \{1,2\}$, let $\theta_i\in \real^3$ drawn uniformly at random from $[0,1]^3$ at the beginning of the experiment. In round $t$, the reward of arm $i\in \{1,2\}$ is $\theta_i^\top s_t + \xi$ where $\xi \sim N(0,1)$ and context $s_t\in \real^3$ is drawn uniformly at random from $[0,1]^3$ with $s_t[0]=1$. 

\subsubsection*{Optimizing the Exploration Rate}

Next, we consider the performance of regret balancing as a bandit model selection strategy. First, consider optimizing the exploration rate in an $\epsilon$-greedy algorithm. The $\epsilon$-greedy is a simple and popular bandit method. In round $t$, the algorithm plays an action chosen uniformly at random with a small probability $\epsilon_t$, and plays the empirically best, or greedy, choice otherwise. For a well-chosen $\epsilon_t$, this simple strategy can be very competitive. The optimal value of $\epsilon_t$ however depends on the unknown reward function: It is known that the optimal value of $\epsilon_t$ is $\min\{1,\frac{5 K}{\Delta^2 t}\}$ where $\Delta$ is the smallest gap between the optimal reward and the sub-optimal rewards~\citep{LS-2020}. By this choice of exploration rate, the regret scales as $\widetilde O(\sqrt{T})$ for $K=2$ and $\widetilde O(T^{2/3})$ for $K>2$.

We apply the regret balancing strategy to find a near-optimal exploration rate. The result directly follows from Theorem~\ref{theorem:model-selection-bandit}. A similar result, but for a different algorithm, is shown by \cite{PPARZLS-2020}.
\begin{corollary}
Let $T$ be the time horizon. Let $B=\{1,2,\dots,\lfloor\log(T)\rfloor\}$. For $i\in B$, let $B_i$ be the $\epsilon$-greedy algorithm with exploration rate $\epsilon_t=2^i/t$ in round $t$. By the choice of $U(t)=t^{1/2}$ for $K=2$ (or $U(t)=t^{2/3}$ for $K>2$), the regret balancing model selection with the set of base algorithms $B$ achieves $\widetilde O(\sqrt{T})$ regret for $K=2$ (or $\widetilde O(T^{2/3})$ for $K>2$). 
\end{corollary}


Next, we evaluate the performance of regret balancing in finding a near optimal exploration rate. Consider a bandit problem with two Bernoulli arms with means $\{0.5,0.45\}$. Consider 18 $\epsilon$-greedy base algorithms with exploration $\epsilon_t = c/t$, where values of $c$ are on a geometric grid in $[1,2T]$. Apply regret balancing with the target regret bound $U(t) = \sqrt{t}$, and the set of $\epsilon$-greedy base algorithms. The experiment is repeated 20 times. Figure~\ref{fig:bandit_model_selection}-Left shows the performance of regret balancing strategy. 

\begin{figure}[h]
\includegraphics[width=0.3\linewidth]{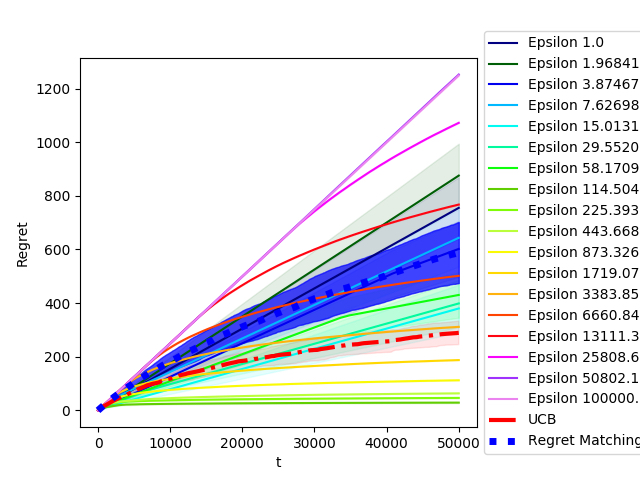} 
\includegraphics[width=0.3\linewidth]{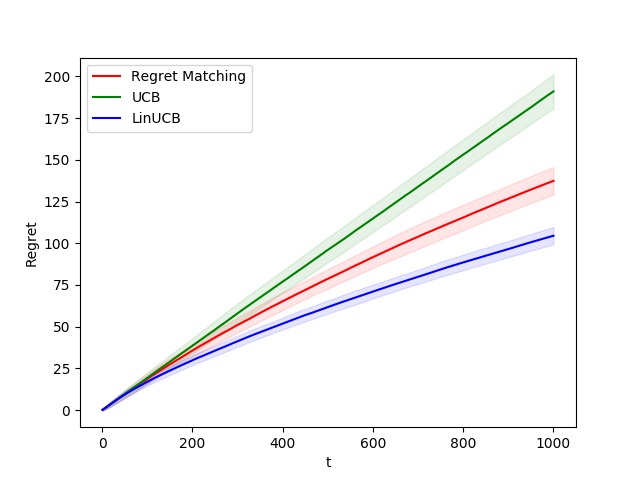} 
\includegraphics[width=0.3\linewidth]{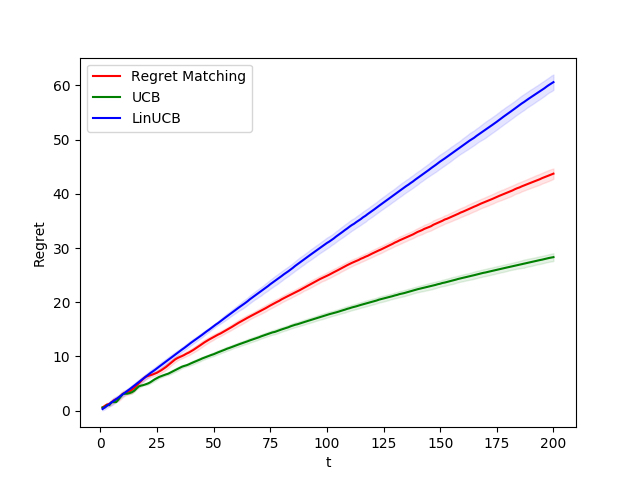}
\caption{Left: Optimising the exploration rate with regret balancing (Mean and standard deviation of 20 runs), Middle and Right: Regret balancing to choose between UCB and LinUCB (Mean and standard deviation of 500 and 200 runs).}
    \label{fig:bandit_model_selection}
\end{figure}

\subsubsection*{Representation Learning}

The sublinear regret bounds of linear bandit algorithms are valid as long as the reward function is truly a linear function of the input feature representation. Assume it is known that the reward function is linear in one of the $M$ feature maps $\{\phi_i:D_t \rightarrow\real^d\,:\, i\in [M]\}$, but the identity of the true feature map is unknown. 
By applying Theorem~\ref{theorem:model-selection-bandit} to $M$ OFUL algorithms, each using one of the feature maps, we obtain a regret that scales as $\widetilde O(M d\sqrt{T})$.

As an application, we consider the problem of choosing between UCB and OFUL. Contexts are drawn from the standard normal distribution, but the first element in the context vector is always 1. The noise is $\xi \sim N(0,\sigma^2=0.1)$. First, consider a problem with $K=2$ arms, each having a reward vector in $\real^{10}$ drawn uniformly at random from $[0,1/3]^{10}$ at the beginning. We use regret balancing with target function $U(t) = \sqrt{2 t}$ to perform model selection between UCB and OFUL. 
Results are shown in Figure~\ref{fig:bandit_model_selection}-Middle. In this experiment, OFUL performs better than UCB, and performance of regret balancing is  in between. Next we consider a problem with $K=5$ arms. Mean reward of arm $i\in [K]$, denoted by $\mu_i$, is generated uniformly at random from $[0,1]$ at the beginning. In each round, we observe a context $s_t \in \real^{10}$, but the expected reward of arm $i$ in each round is $\mu_i$. 
We use target regret $U(t) = \sqrt{5t}$. Figure~\ref{fig:bandit_model_selection}-Right shows that in this setting UCB performs better than OFUL, and performance of regret balancing is again in between.

\subsubsection*{Choosing Among Reinforcement Learning Algorithms}

We consider the model selection problem in finite-horizon reinforcement learning problems. The ideas can be easily extended to average-reward setting as well, but we choose a finite-horizon setting to simplify the presentation.

A finite-horizon reinforcement learning problem is specified by a horizon $H$, a state space $S$ that is partitioned into $H$ disjoint sets, an action space $A$, a transition dynamics $P$ that maps a state-action pair to a distribution over the states in the next stage, and a reward function $r$ that assigns a scalar value to each state-action pair. The objective is to find a policy $\pi$, that is a mapping from states to distributions on actions, that maximizes the total reward. 

\begin{wrapfigure}{}{0.5\textwidth}
\includegraphics[width = \linewidth]{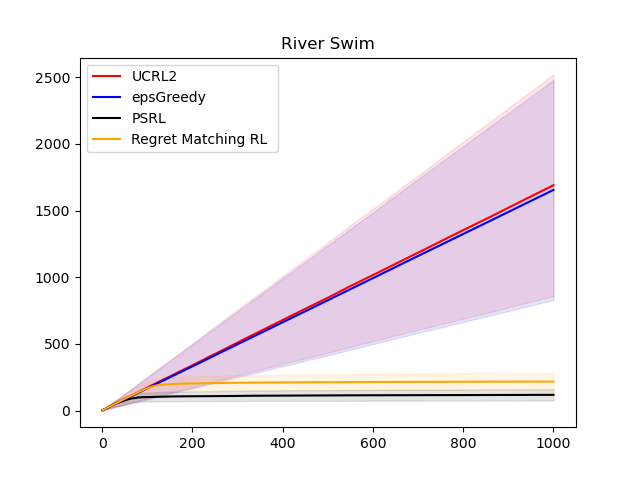} 
\caption{Regret balancing for model selection among $\epsilon$-greedy, UCRL, and PSRL. Mean and standard deviation of 10 runs.}
    \label{fig:rl_exp}
\end{wrapfigure}

The model selection problem is defined next. In episode $t$, the learner chooses base $i_t$ from a set of $M$ RL algorithms, the base is executed for $H$ rounds, and the rewards of the actions are revealed to the learner. Let $V_{*,t}$ be the total reward of the optimal policy in the underlying reinforcement learning problem. Quantities $N_{i,t}$, $R_{i,t}$, $S_{i,t}$, $i_*$, $U$, etc are defined similar to the bandit case. For example, $N_{i,t}$ is the number of episodes that base $i$ is played up to episode $t$. 
The regret balancing strategy is defined next. In episode $t$, let $j_t = \arg\max_{i\in [M]}\, \frac{R_{i,t}}{N_{i,t}} + \frac{U(\delta, S_{i,t})}{N_{i,t}}$ be the optimistic base. 
Let $b_t$ such that $N_{j_t,t} b_t - R_{j_t,t} = U(\delta, S_{j_t,t})$. Define the empirical regret of base $i$ by $\widehat G_{i,t} = N_{i,t} b_t - R_{i,t}$. 
In episode $t$, we choose the base with the smallest empirical regret: $i_t = \arg\min_i\, \widehat G_{i,t}$. The next theorem shows the model selection guarantee for the regret balancing strategy. The analysis is almost identical to the analysis of the bandit model selection in the previous section.
\begin{theorem}
\label{theorem:model-selection-rl}
If $V_{*,t}=V_*$ for a constant $V_*$ regardless of round $t$, and if for any $\delta\in (0,1)$ with probability at least $1-\delta$, $G_{i_*,t} \le U(\delta, S_{i_*,t})$ for any $t$, then $\text{Regret}_T \le M \max_i U(\delta, S_{i,T})$ with probability at least $1-\delta$.
\end{theorem}

In Figure~\ref{fig:rl_exp}, we perform model selection with base algorithms UCRL2~\citep{JOA-2010}, a Q-learning method with $\epsilon$-greedy exploration and $\epsilon=0.1$, and PSRL~\citep{OVR-2013} in the River Swim domain~\citep{strehl2008analysis}. Regret balancing adapts to the best performing strategy (PSRL in this case). 

As another application, consider the problem of choosing state representation in reinforcement learning. Many existing theoretical results hold under the assumption that a correct state representation (or feature map) is given. As examples, \cite{ABBLSW-2019} show sublinear regret bounds under the assumption that the value function of any policy is linear in a given feature vector, while \cite{JYWJ-2019} show sublinear regret bounds for linear MDPs, i.e. when the transition dynamics and the reward function are known to be linear in a given feature vector. Given $M$ candidate feature maps, one of which is fully aligned with the true dynamics of the MDP, we can apply the regret balancing strategy and by Theorem~\ref{theorem:model-selection-rl}, the performance will be optimal up to a factor of $M$. 
\begin{corollary}
\label{cor:RL}
Let $\mathcal{M}= (S, A, H, P, r)$ be a linear MDP parametrized by an unknown feature map $\{ \Phi^*: S \times A\rightarrow \real^d\}$.  Let $F=\{\Phi_i(s,a)\}_{i=1}^M$ be a family of feature maps with $\Phi_i(s,a) \in \mathbb{R}^d$ and  satisfying $\Phi^* \in F$. For regret balancing with target $U(t) = d^{3/2}H^{3/2}T^{1/2}$ and with a class of LSVI-UCB base algorithms~\citep{JYWJ-2019}, each instantiated with a feature map in $F$, the regret is bounded as $\text{Regret}_T \leq \tilde{\mathcal{O}}\left(M\sqrt{  d^3 H^3 T  } \right)$.
\end{corollary}

\cite{MRM-2011, MNOR-2013, OMR-2014, OPLFM-2019} study a closely related but different problem where $M$ state representation functions are given and with at least one such function, the resulting state evolution is Markovian. 


\section{Lower Bounds}

\subsection{Regret Balancing}

In this section we show that for any model selection algorithm there are problem instances where the algorithm must do regret balancing. For simplicity we restrict ourselves to the case $M=2$, and to a simple class of problem instances, although it is possible to extend the argument to richer families and beyond two base algorithms.

Let $\mathcal{M}$ be a model selection algorithm with expected regret $\mathcal{R}(t)$ up to time $t$. We say an algorithm ``model selects" w.r.t. a class of algorithms $\mathcal{B}$ if for any two base algorithms $A, B \in \mathcal{B}$ with expected regret $\mathcal{R}_A$ and $\mathcal{R}_B$, there exists $T_0>0$ such that for all $T \geq T_0$, $\mathcal{R}(T) \leq  \mathcal{O}( \min( \mathcal{R}_A(T), \mathcal{R}_B(T)) )$. We say that algorithm $\mathcal{M}$ is regret balancing for base algorithms $(A, B)$ if for all $\delta \in (0,1)$ there exists $T(\delta)$ such that for all $T \geq T(\delta)$, with probability at least $1-\delta$, 
\begin{equation}
\label{equation::regret_matching_condition}
\log\left( \max\left( \frac{\tilde{\mathcal{R}}_A(T)}{\tilde{\mathcal{R}}_B(T)} , \frac{\tilde{\mathcal{R}}_B(T)}{\tilde{\mathcal{R}}_A(T)} \right) \right) \leq o(\log(T)) \,, 
\end{equation}
where $\tilde{\mathcal{R}}_A(T)$ and $\tilde{\mathcal{R}}_B(T)$ are the empirical regrets of algorithms $A$ and $B$, respectively. 
The main result of this section is to show there exist problem and algorithm classes such that any model selection strategy must be regret balancing. 
\begin{theorem}
\label{thm:lowerbound}
There exists two algorithm classes $\mathcal{B}_1,\mathcal{B}_2$ with $\mathcal{B}_1 \subseteq \mathcal{B}_2$ such that any model selection strategy $\mathcal{M}$ for class $\mathcal{B}_2$ must satisfy the condition in \eqref{equation::regret_matching_condition} for all $\mathcal{A}, B \in \mathcal{B}_1$ whose regrets are distinct.
\end{theorem}
\begin{proof}[Sketch] The complete proof is in Appendix~\ref{subsec:proof-regret-matching-lower}. The proof proceeds by contradiction. We consider a pair of simple deterministic algorithm classes. Suppose there exist two algorithms $A, B \in \mathcal{B}_1$ such that $\mathcal{M}$ does not regret balance them. In this case for infinitely many $T > T(\delta)$ and with probability at least $\delta$ for each such $T$, w.l.o.g $A$'s regret must be larger than that of $B$ by a factor of $T^\beta$ for some $\beta > 0$. We now construct another algorithm $C \in \mathcal{B}_2$ that acts just like $B$ until the moment $\mathcal{M}$ stops pulling it (in the $\delta$ probability event) and then acts optimally. Algorithm $C$ has better regret than $A$. It can be shown that in this $\delta$-probability event, $\mathcal{M}$ will be unable to detect if it's playing $B$ or $C$, thus incurring in a large regret.
\end{proof}

\subsection{The Knowledge of the Optimal Base Regret}

We show that a prior knowledge of the optimal base regret is needed to achieve the optimal regret. 
\begin{theorem}
\label{theorem:lower_bound_optimal_regret}
There is a model selection problem such that if the learner does not know the regret of best base, and does not have access to the arms, then its regret is larger than that of the optimal base.
\end{theorem}
\begin{proof}[Sketch]
The complete proof is in Appendix~\ref{app:lower_bounds}. Let there be two base algorithms, and let $R_1$ and $R_2$ be their regrets incurred when called by the model selection strategy. If $R_1 =o(R_2)$, we can construct the bases such that they both have zero regret after the learner stops selecting them. Therefore their regrets when running alone are $R_1$ and $R_2$, and the learner has regret of the same order as $R_2$, which is higher than the regret of the better base running alone ($R_1$). If however $R_1 \approx R_2$, since the learner does not know the optimal arm reward, we can create another environment where the optimal arm reward is different, so that in the new environment the regrets are no longer equal.
\end{proof}

\section*{Broader Impact}

The work does not present any foreseeable societal consequence.

\bibliography{ref}

\newpage

\appendix
\section{Some useful results}
\label{app:useful}

We state a result on the error of the least-squares predictor.  
\begin{theorem}[Theorem~2 of \cite{APS-2011}]
\label{theorem:least-squares}
Assume $\|\theta_*\|\le S$. Let
\[
\beta_t(\delta) = R \sqrt{\log\left( \frac{ \det(V_t)^{1/2} \det(\lambda I)^{-1/2}}{\delta} \right)} + \lambda^{1/2} S \;.
\]
For any $\delta>0$, with probability at least $1-\delta$, for all $t\ge 0$ and any $x\in \real^d$,
\begin{align}
\label{eq:least-squares}
| x^\top (\widehat \theta_t - \theta_*) | \le \beta_t(\delta) \|x\|_{V_t^{-1}} \;.
\end{align}
\end{theorem}

\begin{lemma}
\label{lemma:sumpotential}
Let $\{X_t\}_{t=1}^\infty$ be a sequence in $\real^d$ and define $V_t = \lambda I + \sum_{k=1}^{t} X_k X_k^\top$ for a regularizer $\lambda \ge 1$. If $\|X_t\|\le 1$ for all $t$, then
\[
\sum_{k=1}^t \|X_k \|_{V_{k-1}^{-1}}^2 \le 2 \log \frac{\det(V_t)}{\det(\lambda I)} \le 2 d \log (1+t/(\lambda d)) \;. 
\]
\end{lemma}

\section{Regret balancing for linear bandits}
\label{app:linear_bandit}


\begin{proof}[Proof of Theorem~\ref{thm:linear_bandit}]
By Theorem~\ref{theorem:least-squares}, with probability at least $1-\delta$, for all $t$, $x_{t,*}^\top \widehat \theta_t + \beta_t(\delta) \|x_{t,*}\|_{V_t^{-1}} \ge x_{t,*}^\top \theta_*$. In what follows, we condition on the high probability event that these inequalities hold.

First, we show that $b_t$ is an optimistic estimate of $x_{t,*}^\top \theta_*$. By definition of $y_t$, 
\begin{align}
\label{eq:bt_lin}
b_t = y_t^\top \widehat \theta_t + \beta_t(\delta) \|y_t\|_{V_t^{-1}} \ge x_{t,*}^\top \widehat \theta_t + \beta_t(\delta) \|x_{t,*}\|_{V_t^{-1}} \ge x_{t,*}^\top \theta_*  \;.
\end{align}
We upper bound the instantaneous regret,
\begin{align*}
r_t &= x_{t,*}^\top \theta_* - x_t^\top \theta_* \\
&\le b_t - x_t^\top \theta_*  &\text{By \eqref{eq:bt_lin}} \\
&\le b_t - x_t^\top \widehat \theta_t + \beta_t(\delta) \|x_t\|_{V_t^{-1}} &\text{By \eqref{eq:least-squares}}\\
&\le \beta_t(\delta) \|x_t\|_{V_t^{-1}} + \|x_t\|_{V_t^{-1}}^2 \left( \frac{b_t - y_t^\top \widehat \theta_t}{\|y_t\|_{V_t^{-1}}^2} \right) &\text{By definition of $x_t$} \\
&= \beta_t(\delta) \|x_t\|_{V_t^{-1}} + \|x_t\|_{V_t^{-1}}^2 \cdot \frac{\beta_t(\delta)}{\|y_t\|_{V_t^{-1}}} &\text{By \eqref{eq:bt_lin}} \;.
\end{align*}
Using the fact that $\lambda_{\text{max}}(V_t) \le \text{trace}(V_t) = \lambda d + \sum_{k=1}^{t-1} \|x_t \|^2 \le \lambda d + t$, and hence $\lambda_{\text{min}}(V_t^{-1}) = \frac{1}{\lambda_{\text{max}}(V_t)} \ge 1/(\lambda d + t)$, we get that $\|y\|_{V_t^{-1}}^2 \ge 1/(\lambda d + t)$ for any $y\in D_t$. Thus,
\[
r_t \le \beta_t(\delta) \|x_t\|_{V_t^{-1}} + \beta_t(\delta) \|x_t\|_{V_t^{-1}}^2 \sqrt{\lambda d + t} \;.
\]
Thus, by Cauchy--Schwarz inequality and Lemma~\ref{lemma:sumpotential},
\begin{align*}
\text{Regret}_T &= \sum_{t=1}^T \left( \beta_t(\delta) \|x_t\|_{V_t^{-1}} + \beta_t(\delta) \|x_t\|_{V_t^{-1}}^2 \sqrt{\lambda d + t} \right) \\
&\le \beta_T(\delta) \left( \sqrt{T \sum_{t=1}^T \|x_t\|_{V_t^{-1}}^2} + 2 d \log (1+T/(\lambda d)) \sqrt{\lambda d + T}) \right) \\
&\le \beta_T(\delta) \left( \sqrt{2 d T \log (1+T/(\lambda d))} + 2 d \log (1+T/(\lambda d)) \sqrt{\lambda d + T}) \right) \;.
\end{align*} 

\end{proof}


\section{Proof of Theorem~\ref{thm:lowerbound}}
\label{subsec:proof-regret-matching-lower}

\begin{proof}

Let $\mathcal{B}_1, \mathcal{B}_2$ be two classes of algorithms defined as follows: if $\mathcal{B} \in \mathcal{B}_1$ then there exists a value $b$ such that $\mathcal{B}$ has a deterministic instantaneous regret of $b$ during all time steps. If $\mathcal{B} \in \mathcal{B}_2$, then there is a time index $t_0$ and two values $b_1$ and $b_2$ such that $\mathcal{B}$ has a deterministic instantaneous regret of $b_1$ for all $t \leq t_0$ and a deterministic instantaneous regret of $b_2$ for all $t > t_0$. We show the following Theorem:

Let $A \in \mathcal{B}_1$ be an algorithm that for all timesteps $t \in [T]$ plays a policy achieving (deterministically) an instantaneous regret of $\frac{1}{T^{1-a}}$ for some $a \in [0,1]$. Similarly let $B \in \mathcal{B}_1$ be an algorithm that for all timesteps $t \in [T]$ plays a policy with a deterministic instantaneous regret of $\frac{1}{T^{1-b}}$ for some $b \in [0,1]$.


We proceed by contradiction. If $\mathcal{M}$ is not regret matching for $(A,B)$, then, there exists an $\epsilon > 0$ such that with probability at least $\epsilon$:
\begin{equation}\label{equation::regret_matching_condition_violated}
\max\left( \frac{\tilde{R}_A(T)}{\tilde{R}_B(T)} , \frac{\tilde{R}_B(T)}{\tilde{R}_A(T)}   \right)  \geq C T^c
\end{equation}
For some nonzero positive constants $C, c > 0$, and for infinitely many $T > T(\epsilon)$. Wlog the condition in Equation \ref{equation::regret_matching_condition_violated} implies that for infinitely many $T  \geq T(\epsilon)$ with probability at least $\epsilon/2$: 
\begin{equation}\label{equation::regret_mismatch}
    \tilde{R}_A(T) \geq C\tilde{R}_B(T) \cdot T^c
\end{equation}
For any such $T$ let this event be called $\mathcal{E}_T$. Define $T_A$ and $T_B$ to be the random number of times in $[T]$ that algorithm $A$ (respectively algorithm $B$) was called by $\mathcal{M}$. In this case, Equation \ref{equation::regret_mismatch} becomes, with probability at least $\frac{\epsilon}{2}$:
\begin{equation}\label{equation::initial_consequence_regretmismatch}
     C T_b \frac{1}{T^{1-b}} T^c \leq T_a \frac{1}{T^{1-a}} 
\end{equation}

Which in turn implies $T_a \geq  CT^cT_b T^{b-a} $, additionally since $T_a \leq T$,  with probability at least $\epsilon/2$ we have    $T_b  \leq \frac{1}{C}  T^{1+a -b-c} $. We now proceed to show a lower bound for the regret of the master in each of two cases, $a > b$ and $b > a$. 

\paragraph{Case $a > b$} 
Let $\mathcal{E}_T = \mathcal{E}_T^1 \cup \mathcal{E}_T^1$ where $\mathcal{E}_T^1 = \{ T_a \geq \frac{T}{2}\} \cap \mathcal{E}_T$ and  $\mathcal{E}_T^1 = \{ T_a < \frac{T}{2}\} \cap \mathcal{E}_T$. Notice that $\max(\mathbb{P}(\mathcal{E}_T^1), \mathbb{P}(\mathcal{E}_T^2) ) \geq \frac{\epsilon}{4}$. In $\mathcal{E}_T^1$ we have $\tilde{R}_a(T) \geq \frac{T^a}{2}$. In $\mathcal{E}_T^2$, $T_b \geq \frac{T}{2}$ which in turn implies by Equation \ref{equation::initial_consequence_regretmismatch} that $T_a \geq  CT^{1+ c + b-a} $ and therefore that in $\mathcal{E}_T^2$ it holds that $\tilde{R}_a(T) \geq C \frac{T^{c+b}}{2}$. Since $\mathcal{R}(T) = E[\tilde{R}_a(T) + \tilde{R}_b(T)]$, we conclude that $\mathcal{R}(T) \geq \frac{\epsilon}{4}\min\left(C \frac{T^{c+b}}{2} , \frac{T^a}{2} \right)$.

\textbf{Case $b > a$}. Assume $\mathcal{M}$ has model selection guarantees (in expectation ) w.r.t. algorithm $A$. Therefore $\mathcal{R}(T) \leq C'' T^a$. As a consequence of Equation \ref{equation::initial_consequence_regretmismatch} with probability at least $\frac{\epsilon}{2}$ we it holds that $T_b  \leq \frac{1}{C}  T^{1+a -b-c}  = o(T)$. 

This analysis shows that in case $\mathcal{M}$ does not satisfy regret matching, then it must be the case that:

\begin{enumerate}
        \item If $a > b$: Then $\mathcal{M}$ must incur in an expected regret of at least $\frac{\epsilon}{2}\min\left(C \frac{T^{c+b}}{2} , \frac{T^a}{2} \right)$ for some $c > 0$. Thus already precluding any model selection guarantees for $\mathcal{M}$. 
    \item If $b > a$: Then with probability at least $\frac{\epsilon}{2}$ it follows that $T_b \leq \frac{1}{C}  T^{1+a -b-c}$ for some constants $C, c$. Furthermore, if $\mathcal{M}$ is assumed to satisfy model selection guarantees, it must be the case that for $T$ large enough, we can conclude that with probability at least $\frac{\epsilon}{2}$, $T_a \geq T/2$. We focus on this case to find a contradiction.
\end{enumerate}

\textbf{Two alternative worlds} Having analyzed what happens if a master does not do regret matching with algorithms $A$ and $B$, we proceed to show our lower bound. Let $(A,B)$ two base algorithms defined as above and let $(A',B')$ two base algorithms defined as:
\begin{enumerate}
    \item $A'$ acts exactly as a does.
    \item $B'$ acts as $B$ does only up to time $t' = \min( \frac{1}{C} T^{1+a-b -c}+1, T)$, afterwards it pulls the optimal arm (deterministically).
\end{enumerate}
Let $T_{a'}$ and $T_{b'}$ be the random number of times $A'$ and $B'$ are played by $\mathcal{M}$.

Suppose the master $\mathcal{M}$ above is presented with $(A'', B'')$ sampled uniformly at random between $(A, B)$ and $(A', B')$. We show the following:




Let $b > a$. Note that environment $(A', B')$ is indistinguishable from environment $(A,B)$ in the probability at least $\epsilon/2$ event that $T_b < t'$. This implies that in environment $(a', b')$, and with probability at least $\frac{\epsilon}{2}$, $T_{b'} < t' = o(T)$. In this same event and for $T$ large enough since $T_{a'} + T_{b'} = T$ it must be the case that $T_{a'} \geq T/2$ and $T_{b'} \leq \frac{1}{C}  T^{1+a -b-c} $, and therefore that:
\begin{equation*}
     E_{(A'',B'')} [ \mathcal{R}(T) | (A'', B'') = (A', B')] \geq \frac{\epsilon}{8} T^{a}
\end{equation*}
Since for $T$ large enough the optimal regret for $(A', B')$ is instead $\frac{1}{C}T^{1+a-b-c}*\frac{1}{T^{1-b}} = \frac{1}{C} T^{a-c}$, and for $T$ large enough:
\begin{equation*}
    \frac{1}{C}T^{a-c} = o(T^{a})
\end{equation*}
We conclude that $\mathcal{M}$ couldn't have possibly satisfied model selection.
\end{proof}

\section{Proof of Theorem~\ref{theorem:lower_bound_optimal_regret}}
\label{app:lower_bounds}

\begin{proof}
Let the set of arms be $\{a_1,a_2,a_3\}$. Let $x$ and $y$ be such that $0< x < y \le 1$. Let $\Delta = T^{x-1 + (y-x)/2}$. Define two environment $\mathcal{E}_1$ and $\mathcal{E}_2$ with reward vectors $\{1,1,0\}$ and $\{1+\Delta,1,0\}$, respectively. Let $B_1$ and $B_2$ be two base algorithms defined by the following fixed policies when running alone in $\mathcal{E}_1$ or $\mathcal{E}_2$: 
\[
\pi_1 = 
\begin{cases}
    a_2       & \quad \text{w.p. } 1-T^{x-1}\\
    a_3  & \quad \text{w.p. } T^{x-1}
\end{cases}
\,,\qquad
\pi_2 = 
\begin{cases}
    a_2       & \quad \text{w.p. } 1-T^{y-1}\\
    a_3  & \quad \text{w.p. } T^{y-1}
\end{cases} \;.
\]
We also construct base $B'_2$ defined as follows. Let $c_2 > 0$ and $\epsilon_2 = (y-x)/4$ be two constants. Base $B'_2$ mimics base $B_2$ when $t\le c_2 T^{x - y+1 + \epsilon_2 }$, and picks arm $a_1$ when $t>  c_2 T^{x - y+1 + \epsilon_2}$. The instantaneous rewards of $B_1$ and $B_2$ when running alone are $r^1_t = 1-T^{x-1}$ and $r^2_t = 1-T^{y-1}$ for all $1 \le t \le T$. Next, consider model selection with base algorithms $B_1$ and $B_2$ in $\mathcal{E}_1$. Let $T_1$ and $T_2$ be the number of rounds that $B_1$ and $B_2$ are chosen, respectively. 

First, assume case (1): 
There exist constants $c > 0$, $\epsilon > 0$, $p\in (0,1)$, and $T_0>0$ such that with probability at least $p$, $T_2 \ge c T^{x - y+1+ \epsilon}$ for all $T>T_0$. 


The regret of base $B_1$ when running alone for $T$ rounds is $T \cdot T^{x-1} = T^{x}$. The regret of the model selection method is at least 
\[
p\cdot T_2 \cdot T^{y-1} \ge p\cdot c T^{x - y+1+ \epsilon}\cdot T^{y-1}= p\cdot c \cdot T^{x+\epsilon} \;.
\]
Given that the inequality holds for any $T > T_0$, it proves the statement of the lemma in case (1). 

Next, we assume the complement of case (1): 
For all constants $c > 0$, $\epsilon > 0$, $p\in (0,1)$, and $T_0>0$, with probability at least $1-p$, $T_2 < c T^{x - y+1+ \epsilon}$ for some $T>T_0$. 

Let $T$ be any such time horizon. Consider model selection with base algorithms $B_1$ and $B'_2$ in environment $\mathcal{E}_2$ for $T$ rounds. Let $T'_1$ and $T'_2$ be the number of rounds that $B_1$ and $B'_2$ are chosen. Given the black-box interaction model, the fact that $B_2$ and $B'_2$ behave the same for $T_2 < c T^{x - y+1+ \epsilon}$, and that $B_1$ and $B_2$ never choose action $a_1$, with probability $p > 1/2$,  $T'_2 < c_2 T^{x - y+1 + \epsilon_2}$ and $T'_1 > T/2$, and the model selection strategy behaves the same as when it runs $B_1$ and $B'_2$ in $\mathcal{E}_2$.

In environment $\mathcal{E}_2$, the regret of base $B'_2$ when running alone for $T$ rounds is bounded as
\[
(\Delta +  T^{y-1})  c_2 T^{x - y+1 + \frac{y-x}{4} } =  c_2 T^{\frac{5x-y}{4} } + c_2T^{\frac{3x+y}{4} }< 2c_2 T^{ \frac{3x+y}{4} } 
\]
Given that with probability $p>1/2$, $T'_1 > T/2$, the regret of the learner is lower bounded as,
\[
p (\Delta + T^{x-1}) \cdot \frac{T}{2} >\frac{1}{2}(T^{x-1+\frac{y-x}{2}}+T^{x-1} )\cdot \frac{T}{2}  < \frac{1}{2} T^{\frac{x+y}{2}}\,, 
\]
which is larger than the regret of $B'_2$ running alone because $ \frac{3x+y}{4} <\frac{x+y}{2}$. The statement of the lemma follows given that for any $T_0$ there exists $T>T_0$ so that the model selection fails.
\end{proof}

\end{document}